\documentclass[sigconf,nonacm]{acmart}
\setkeys{acmart.cls}{balance=false}
\usepackage[ruled,vlined]{algorithm2e}
\usepackage{colortbl}
\usepackage{multirow}
\usepackage{siunitx}
\usepackage{placeins}

\newcommand{\result}[2]{\num{#1}\,(\num[retain-explicit-plus]{#2}\%)}
\newcommand{\percent}[1]{\num[retain-explicit-plus]{#1}\%}
\newcommand{\bluecell}[1]{\cellcolor{highlightcyan}#1}
\newcommand{\graycell}[1]{\cellcolor{gray!10}#1}

\definecolor{gray20}{gray}{0.90}
\definecolor{gray10}{gray}{0.95}
\definecolor{highlightcyan}{HTML}{D7F6FF}
\newtheoremstyle{boldproposition}
  {\topsep}{\topsep}{\itshape}{} {\bfseries}{.}{0.5em}{}
\theoremstyle{boldproposition}
\newtheorem{proposition}{Proposition}[section]

\setcopyright{none}
\makeatletter
\let\hyxmp@parse@acmart\relax
\makeatother

\begin{document}

\title[PhaseBreak: Proactive Temporal Protection]{PhaseBreak: Proactive Temporal Protection against Unauthorized Image-to-Video Animation}

\author{Jiashu Zhang}
\affiliation{%
  \institution{Jilin University}
  \city{Changchun}
  \country{China}
}
\email{zhangjiashu0820@gmail.com}

\author{Xuanyou Li}
\affiliation{%
  \institution{Jilin University}
  \city{Changchun}
  \country{China}
}
\email{2719226114@qq.com}

\author{Qixin Zhang}
\authornote{Corresponding author.}
\affiliation{%
  \institution{Wuhan University}
  \city{Wuhan}
  \country{China}
}
\email{qixin.zhang2026@gmail.com}
\renewcommand{\shortauthors}{Jiashu Zhang et al.}

\begin{abstract}
  Image-to-Video (I2V) diffusion models can animate a single portrait or
  photograph, enabling shared images to be repurposed for motion synthesis
  without the subject's authorization. We study proactive visual protection
  against such use: before an image is released, its owner applies a bounded
  perturbation that preserves the sharing workflow but disrupts downstream
  motion generation. Existing spectral disruption objectives globally
  attenuate temporal representations and therefore provide limited control
  over subtle or spatially localized motion. Our key idea is that generated
  dynamics depend not only on temporal feature strength, but also on the order
  in which feature events occur. We propose \emph{PhaseBreak}, a white-box
  spectral retiming protection method that matches intermediate denoiser
  features to a constrained temporal rearrangement of their clean counterparts.
  PhaseBreak operates along the true temporal axis and uses a subject-aware mask
  to confine the protective perturbation to the depicted subject. Evaluations
  on Stable Video Diffusion and DynamiCrafter show that protected images induce
  greater short-term motion variation and lower temporal regularity, impeding
  stable motion customization from the released image.
\end{abstract}

\keywords{image-to-video generation, proactive image protection, visual privacy, temporal dynamics, spectral optimization}

\maketitle
\hypersetup{pdfauthor={Jiashu Zhang, Xuanyou Li, Qixin Zhang},
  pdfsubject={Proactive image protection against unauthorized image-to-video animation},
  pdfpublisher={}}
\raggedbottom

\section{Introduction}
Image-to-Video (I2V) diffusion models animate a single reference image into a
temporally coherent video by combining image conditioning with learned video
priors~\cite{ho2022videodiffusion,blattmann2023stablevideo,xing2023dynamicrafter}.
Recent systems have improved open-domain animation and generation efficiency,
making reference-conditioned video synthesis increasingly
accessible~\cite{xing2023dynamicrafter,hacohen2024ltxvideo}. Although this
interface lowers the barrier to video creation, it also allows a portrait or
personal photograph to be converted into customized actions without the
depicted subject's authorization. Once an image has been shared, its pixels can
be supplied directly to an I2V generator, while the image owner has little
control over the motion subsequently assigned to the subject. This setting
creates a proactive visual-privacy problem: the image should remain suitable
for ordinary sharing but resist unauthorized animation.

Proactive protection has been studied for unauthorized subject-driven
image synthesis. Anti-DreamBooth perturbs released photographs to interfere
with subsequent subject customization~\cite{vanle2023antidreambooth}, while
SimAC improves anti-customization through timestep and intermediate-feature
selection~\cite{wang2024simac}. IDProtector extends this line to zero-shot,
encoder-based identity-preserving generation~\cite{song2025idprotector}.
However, these methods target the appearance or identity learned by
text-to-image systems. Protecting an image against motion customization is
different: the perturbation must influence an ordered video representation
from a single conditioning frame, and success depends on disrupting generated
dynamics rather than only degrading subject appearance.

A protective perturbation must therefore survive the strong spatio-temporal
prior of the video denoiser. Conventional adversarial objectives constrain
pixel-space changes~\cite{goodfellow2015explaining,madry2018towards}, but do not
explicitly model temporal behavior and may be absorbed by denoising or expressed
mainly as superficial artifacts. This difficulty is most pronounced for
localized actions whose temporal evidence occupies only a limited part of the
intermediate representation.

Spectral disruption methods such as MEDUSA~\cite{yu2026medusa} suppress motion
by attenuating temporal attention toward a low-rank state. Although this
mechanism can reduce pronounced motion, global attenuation does not specify how
fine-grained events should be displaced in time. Weak or localized events may
instead be erased together with useful temporal content. A proactive protection
mechanism therefore requires a more targeted objective that makes the released
image unreliable for motion customization without relying on uniform temporal
suppression.

Intermediate denoiser features retain an ordered trace of temporal events.
Their phase and ordering encode when local changes occur, whereas their
magnitude reflects the strength of those changes. Building on this observation,
our key insight is to protect an image by matching its induced representation
to a nearby temporal reordering of the clean feature sequence. Shifting feature
events to neighboring positions disrupts the temporal organization required for
stable animation while preserving the feature content used to construct the
protection target.

We propose \emph{PhaseBreak}, a white-box spectral retiming framework for
proactive I2V protection. Before releasing an image, the owner constructs a
bounded protective perturbation by matching intermediate denoiser features to a
constrained retimed target. PhaseBreak explicitly preserves the architecture's
true temporal axis and combines time-domain retiming with spectral phase
matching. A subject-aware mask confines both the perturbation and its gradient
to the depicted subject, preventing direct modification of unrelated background
regions.

\textbf{Our contributions are as follows:}
\begin{itemize}
  \item \textbf{Protective Retiming.} We formulate temporal retiming as a
  proactive protection objective for I2V diffusion models. It changes the order
  of local temporal events rather than globally attenuating temporal attention,
  targeting the dynamics used for unauthorized animation.
  \item \textbf{PhaseBreak Framework.} We develop a spectral protection method
  that preserves the true temporal axis while matching a constrained retimed
  target under a bounded, subject-localized image perturbation.
  \item \textbf{Cross-Model Evaluation.} We evaluate the protection on SVD and
  DynamiCrafter using 355 VBench-I2V
  examples \cite{huang2024vbench,huang2026vbenchpp} and complementary
  measurements of motion variation and temporal regularity.
\end{itemize}

\section{Related Work}
\noindent\textbf{Video Diffusion and Image-to-Video Generation.}
Video diffusion models extend denoising-based generation from images to
spatio-temporal data by modeling temporal dependencies alongside appearance
and motion~\cite{ho2022videodiffusion,xing2023videodiffusionsurvey}. Early
large-scale T2V systems include Make-A-Video, Imagen Video, and
ModelScopeT2V~\cite{singer2022makeavideo,ho2022imagenvideo,wang2023modelscope}.
VideoCrafter1 and I2VGen-XL extend open video diffusion toward image
conditioning~\cite{chen2023videocrafter1,zhang2023i2vgenxl}. AnimateDiff adds
transferable motion modules to personalized image models~\cite{guo2023animatediff},
whereas SVD and DynamiCrafter provide general I2V
backbones~\cite{blattmann2023stablevideo,xing2023dynamicrafter}. Meanwhile,
LTX-Video emphasizes efficient latent generation~\cite{hacohen2024ltxvideo}.
These models propagate image conditioning through features that jointly encode
spatial content and temporal evolution.

\noindent\textbf{Adversarial and Protective Perturbations.}
Prior attacks on diffusion-based generators have targeted either the generated
semantics or the robustness of the diffusion process. AdvDiff and Mist use
diffusion models to construct adversarial examples, while semantic attacks
optimize perturbations toward high-level changes rather than isolated pixel
errors~\cite{dai2024advdiff,liang2023mist,wang2023semanticadversarial}. Prior
work also studies attacks against diffusion-based purification and image-to-image
diffusion pipelines~\cite{kang2023diffattack,zeng2024advi2i}. The closest
proactive image defenses disrupt editing, customization, or identity-preserving
generation~\cite{salman2023photoguard,vanle2023antidreambooth,
wang2024simac,song2025idprotector}. For I2V protection, I2VGuard targets misuse
of conditioning images, while Vid-Freeze and MEDUSA directly manipulate
temporal generation behavior~\cite{gui2025i2vguard,chowdhury2025vidfreeze,
yu2026medusa}. MFA provides a complementary diffusion attack based on
timestep vulnerability and mean fluctuation~\cite{yu2024mfa}. MEDUSA, the
closest spectral approach, suppresses temporal behavior by attenuating a
spectral representation of temporal attention.
Beyond protective perturbations, TEAR exposes safety failures arising from
temporal event composition, and T2VAttack evaluates prompt attacks against
semantic and temporal behavior~\cite{he2026tear,li2025t2vattack}. Anti-I2V
extends proactive image protection across video diffusion backbones, whereas
practical evaluations of Stable Diffusion protection show that purification
and deployment conditions can weaken protective noise~\cite{vu2026antii2v,
zhao2024protective}.
This mechanism also suggests that adversarial optimization can be repurposed
as proactive protection before an image is released. However, its global
operation is less suited to localized, fine-grained motion because
it reduces temporal response strength without specifying when a local event
should occur. PhaseBreak instead changes the ordering of intermediate temporal
events through constrained retiming and uses a subject-aware mask to confine
the protective perturbation to the intended subject.

\begin{figure*}[t]
  \centering
  \includegraphics[width=0.80\textwidth]{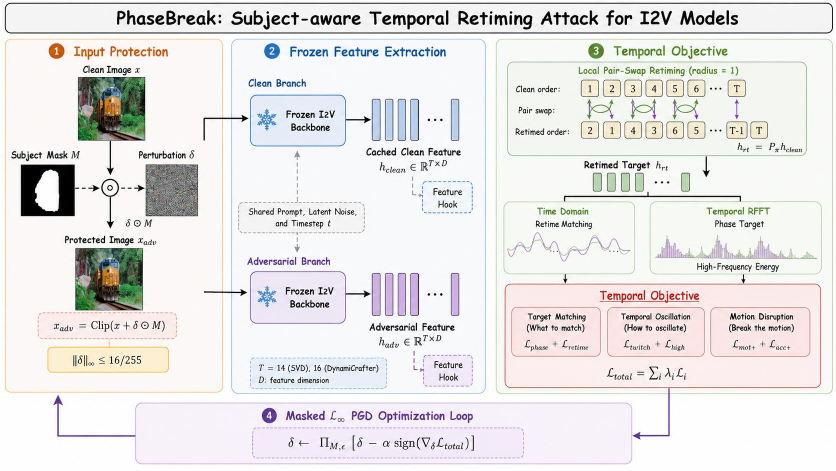}
  \caption{\textbf{Overview of PhaseBreak.} Frozen I2V backbones extract clean
  and protected features, while a subject mask constrains the perturbation.
  Temporal and spectral objectives match the protected features to a locally
  retimed clean target.}
  \Description{Overview of the PhaseBreak pipeline. The diagram contains four
  connected stages: subject-masked input protection, clean and adversarial
  feature extraction with frozen image-to-video backbones, construction of a
  local temporal-retiming target with temporal and spectral objectives, and a
  masked projected-gradient optimization loop.}
  \label{fig:phasebreak-overview}
\end{figure*}

\section{Preliminaries and Problem Formulation}
\subsection{Image-to-Video Diffusion and Temporal Attention}
Let $\mathbf{x}\in[0,1]^{C\times H\times W}$ denote a reference image and
$p$ the text prompt when the model uses text conditioning. An Image-to-Video
(I2V) diffusion model generates
a video $\mathbf{v}\in\mathbb{R}^{T_v\times C\times H\times W}$ by conditioning
its iterative denoising process on $\mathbf{x}$ and, when supported,
$p$~\cite{ho2020ddpm,song2021ddim,rombach2022latent,
ho2022videodiffusion}. At diffusion step
$t$, the denoiser predicts the update of a noisy video latent $\mathbf{z}_t$,
which we write as
\begin{equation}
  \widehat{\boldsymbol{\epsilon}}_t
  = \epsilon_{\theta}\!\left(
      \mathbf{z}_t,t,c_{\mathrm{img}}(\mathbf{x}),c_{\mathrm{txt}}(p)
    \right),
  \label{eq:i2v-denoiser}
\end{equation}
where $\epsilon_{\theta}$ is frozen, $c_{\mathrm{img}}$ encodes the image, and
$c_{\mathrm{txt}}$ is omitted without text conditioning. Repeated denoising and
latent decoding define
$\mathbf{v}=G_{\theta}(\mathbf{x},p;\boldsymbol{\xi})$, where
$\boldsymbol{\xi}$ collects the initial noise and other sampling randomness.

For a denoiser layer $\ell$, we denote its intermediate activation at step
$t$ by
\begin{equation}
  \mathbf{h}_{\ell,t}(\mathbf{x})
  = \phi_{\ell,t}\!\left(
      \mathbf{z}_t,c_{\mathrm{img}}(\mathbf{x}),c_{\mathrm{txt}}(p)
    \right),
  \label{eq:intermediate-feature}
\end{equation}
where $\phi_{\ell,t}$ represents the computation up to layer $\ell$. After
exposing the temporal axis and flattening the remaining spatial axes, we write
$\mathbf{h}_{\ell,t}(\mathbf{x})\in
\mathbb{R}^{T_h\times N\times d}$, where $T_h$ is the feature-level temporal
length, $N$ is the number of spatial tokens per frame, and $d$ is the channel
dimension.

\textbf{Temporal Attention.}
At spatial position $n$, temporal-attention tokens form
$\mathbf{h}_{\ell,t}^{(n)}\in\mathbb{R}^{T_h\times d}$. Learned projections
produce $\mathbf{Q}^{(n)}$, $\mathbf{K}^{(n)}$, and $\mathbf{V}^{(n)}$; their
scaled dot-product interaction~\cite{vaswani2017attention} is
\begin{equation}
  \resizebox{0.84\columnwidth}{!}{$\displaystyle
  \operatorname{Attn}_{\mathrm{temp}}\!\left(
    \mathbf{h}_{\ell,t}^{(n)}(\mathbf{x})
  \right)
  = \operatorname{softmax}\!\left(
      \frac{\mathbf{Q}^{(n)}(\mathbf{K}^{(n)})^{\top}}{\sqrt{d_k}}
    \right)\mathbf{V}^{(n)}$},
  \label{eq:temporal-attention}
\end{equation}
The resulting $T_h\times T_h$ attention matrix aggregates information across
frames. Both evaluated architectures retain a dedicated temporal dimension,
although its tensor index differs.

These activations jointly encode spatial content and temporal evolution. We use
$\mathcal{H}(\mathbf{x})=\{\mathbf{h}_{\ell,t}(\mathbf{x}):
\ell\in\mathcal{L},t\in\mathcal{T}\}$ to denote the collection of features
from the selected layers $\mathcal{L}$ and diffusion steps $\mathcal{T}$.
Section~\ref{sec:phasebreak} specifies the architecture-dependent temporal
dimension.

\subsection{Motivation and Problem Formulation}
\label{sec:problem-formulation}
\textbf{Empirical Observation.}
Under matched sampling, locally permuting a clean feature sequence preserves
its vectors but changes adjacent differences and generally its Fourier phase.
This separates temporal order from global feature strength and motivates
matching a nearby retiming of the clean representation.

\textbf{Theoretical Formulation.}
Flatten the non-temporal dimensions of a feature into
$\mathbf{H}\in\mathbb{R}^{T_h\times D}$, and let
$\mathbf{P}_{\pi}$ be the permutation matrix induced by a temporal schedule
$\pi$. The following proposition establishes that retiming preserves feature
content while providing a controllable target for temporal variation.

\begin{proposition}[Content Preservation under Retiming]
\label{prop:retiming-preservation}
For $\widetilde{\mathbf{H}}=\mathbf{P}_{\pi}\mathbf{H}$,
the temporal mean and feature energy are preserved. Moreover, if
$\mathbf{H}^{p}$ approximates $\widetilde{\mathbf{H}}$ and $\mathbf{D}$ is the
first-difference operator, then
\begin{equation}
  \begin{gathered}
  \mathbb{E}_{T_h}[\widetilde{\mathbf{H}}]
    =\mathbb{E}_{T_h}[\mathbf{H}],\qquad
    \|\widetilde{\mathbf{H}}\|_F=\|\mathbf{H}\|_F,\\
  \|\mathbf{D}\mathbf{H}^{p}-\mathbf{D}\widetilde{\mathbf{H}}\|_F
    \leq 2\|\mathbf{H}^{p}-\widetilde{\mathbf{H}}\|_F.
  \end{gathered}
  \label{eq:retiming-preservation}
\end{equation}
\end{proposition}

\begin{proof}
Since $\mathbf{P}_{\pi}^{\top}\mathbf{P}_{\pi}=\mathbf{I}$ and
$\mathbf{1}^{\top}\mathbf{P}_{\pi}=\mathbf{1}^{\top}$, permutation preserves
the norm and mean. Submultiplicativity gives
$\|\mathbf{D}(\mathbf{H}^{p}-\widetilde{\mathbf{H}})\|_F
\leq\|\mathbf{D}\|_2
\|\mathbf{H}^{p}-\widetilde{\mathbf{H}}\|_F$,
and $\|\mathbf{D}\|_2\leq2$ yields Eq.~\eqref{eq:retiming-preservation}.
\end{proof}

Proposition~\ref{prop:retiming-preservation} provides the theoretical basis of
PhaseBreak. A constrained permutation changes when feature events occur
without removing their content, while reducing the distance to the retimed
target bounds the corresponding first-order dynamic error. This result leads
to the time-domain retiming loss; because temporal order is also expressed by
Fourier phase, PhaseBreak additionally aligns the phase of the protected
feature with that of the same retimed target.

\textbf{Protection Setting.}
We consider an image owner who applies PhaseBreak before sharing a reference
image. The owner has white-box access to the anticipated I2V generator,
including its parameters, sampling procedure, and intermediate denoiser
features, and can differentiate through the image-conditioning and denoising
paths. The protected image is then released without modifying the downstream
model or sampler. During optimization, the clean and protected branches share
the prompt and sampling randomness $\boldsymbol{\xi}$, so their difference is
attributable to the protective image perturbation. The objective is to prevent
the released image from supporting stable, unauthorized motion synthesis.

\textbf{Perturbation Constraint.}
Given a clean image $\mathbf{x}$, the owner releases
$\mathbf{x}^{\mathrm{prot}}=\mathbf{x}+\boldsymbol{\delta}$, where
$\|\boldsymbol{\delta}\|_{\infty}\leq\varepsilon$ and
$\mathbf{x}^{\mathrm{prot}}\in[0,1]^{C\times H\times W}$. A subject mask
$\mathbf{M}$ restricts admissible changes to the depicted subject; we denote
this bounded, masked feasible set by
$\mathcal{S}(\mathbf{x},\varepsilon,\mathbf{M})$. The projection used to
enforce these constraints is specified in Section~\ref{sec:subject-mask}.

\textbf{Optimization Objective.}
The protection seeks to disrupt unauthorized motion synthesis under a bounded,
spatially localized image perturbation. We express this goal through a
differentiable feature objective $\mathcal{J}$ evaluated between protected and
clean intermediate representations:
\begin{equation}
  \boldsymbol{\delta}^{\star}
  = \arg\min_{\boldsymbol{\delta}\in
      \mathcal{S}(\mathbf{x},\varepsilon,\mathbf{M})}
    \frac{1}{|\mathcal{L}||\mathcal{T}|}
    \sum_{\ell\in\mathcal{L}}\sum_{t\in\mathcal{T}}
    \mathcal{J}\!\left(
      \mathbf{h}_{\ell,t}(\mathbf{x}+\boldsymbol{\delta}),
      \mathbf{h}_{\ell,t}(\mathbf{x})
    \right).
  \label{eq:phasebreak-objective}
\end{equation}
The clean representation is fixed, and gradients pass only through the
protected branch. Section~\ref{sec:phasebreak} instantiates $\mathcal{J}$ with
a constrained temporal-retiming target; no target video or frame-level motion
annotation is required.

\section{PhaseBreak}
\label{sec:phasebreak}
\subsection{Overview}
Figure~\ref{fig:phasebreak-overview} summarizes the pipeline that instantiates
Eq.~\eqref{eq:phasebreak-objective}. PhaseBreak identifies the temporal axis of
each hooked feature, constructs a constrained retimed clean target, and matches
the protected feature to it with temporal and spectral losses. Auxiliary
dynamics terms strengthen short-lived changes, while masked projected updates
enforce the image-space constraint. Alg.~\ref{alg:phasebreak} gives the complete
optimization.

\subsection{Subject-Localization Mask}
\label{sec:subject-mask}
\textbf{Image-space Localization.}
To enforce the feasible set introduced in
Section~\ref{sec:problem-formulation},
PhaseBreak clamps the aligned subject mask $\mathbf{M}$ to $[0,1]$, broadcasts
it across image channels, and applies it to both the perturbation and its
gradient. For an unconstrained update $\widehat{\boldsymbol{\delta}}$, the
masked projection is
\begin{equation}
  \Pi_{\mathcal{S}}(\widehat{\boldsymbol{\delta}})
  = \mathbf{M}\odot\left(
      \operatorname{clip}_{[0,1]}\!\left(
        \mathbf{x}+\operatorname{clip}_{[-\varepsilon,\varepsilon]}(
          \mathbf{M}\odot\widehat{\boldsymbol{\delta}})
      \right)-\mathbf{x}
    \right).
  \label{eq:masked-projection}
\end{equation}
The first multiplication prevents an update outside the localized region, the
$\ell_\infty$ clipping enforces the protection budget, and the image clipping
keeps the perturbed reference in its valid range. The final multiplication
removes any value outside the mask after image-range projection. When no mask
is supplied, $\mathbf{M}$ is an all-ones tensor and Eq.~\eqref{eq:masked-projection}
reduces to the standard image-space projection.

The mask acts only in image space; the temporal loss is evaluated on the full
intermediate activation. Consequently, the mask determines where the
protected image can change, while temporal supervision remains defined by the
selected layer and diffusion timestep.

\subsection{Architecture-Specific Temporal Axis}
\textbf{Preserving the Temporal Axis.}
PhaseBreak supplies the temporal dimension explicitly to every temporal
operator. For a denoiser architecture $m$, let
$d_{\mathrm{temp}}^{(m)}$ denote the temporal axis of its intermediate feature
$\mathbf{h}^{(m)}$. Its temporal length is defined by
\begin{equation}
  T_h^{(m)}=\operatorname{size}\!\left(
    \mathbf{h}^{(m)},d_{\mathrm{temp}}^{(m)}
  \right).
  \label{eq:model-temporal-axes}
\end{equation}
The clean and protected branches use the same $d_{\mathrm{temp}}^{(m)}$.
FFT, retiming, first differences, and second differences operate only along
this axis; all remaining dimensions are reduced only when forming scalar
losses.

\begin{algorithm}[!t]
\footnotesize
\caption{Optimization of PhaseBreak}
\label{alg:phasebreak}
\KwIn{$\mathbf{x},p,\mathbf{M},\varepsilon,\alpha,K,\mathcal{L},\mathcal{T}$}
\KwOut{Protected reference image $\mathbf{x}^{\mathrm{prot}}$}
Encode $p$; sample the shared latent; prepare $\mathbf{M}$\;
Cache detached clean features $\{\mathbf{h}^{c}_{\ell,t}\}$\;
Initialize $\boldsymbol{\delta}_0=\mathbf{0}$\;
\For{$k\leftarrow 0$ \KwTo $K-1$}{
  $\mathbf{x}^{\mathrm{prot}}_k\leftarrow
  \operatorname{clip}_{[0,1]}(\mathbf{x}+\boldsymbol{\delta}_k)$\;
  \ForEach{$t\in\mathcal{T}$ and $\ell\in\mathcal{L}$}{
    Extract $\mathbf{h}^{p}_{\ell,t}$ along the verified temporal axis\;
    Construct $\boldsymbol{\pi}$ and $\widetilde{\mathbf{h}}^{c}_{\ell,t}$
    under $(r,B_0,B_1)$\;
    Accumulate $\mathcal{J}_{\ell,t}$\;
  }
  Average the losses and update $\boldsymbol{\delta}_k$ using
  Eq.~\eqref{eq:masked-pgd-update}\;
}
$\mathbf{x}^{\mathrm{prot}}\leftarrow
\operatorname{clip}_{[0,1]}(\mathbf{x}+\boldsymbol{\delta}_K)$\;
\Return{$\mathbf{x}^{\mathrm{prot}}$}\;
\end{algorithm}

\subsection{Constrained Temporal Retiming}
\textbf{Local Retiming Schedule.}
Let $\mathbf{h}^{c}$ and $\mathbf{h}^{p}$ denote a clean intermediate feature
and its protected counterpart, and let $T_h$ be their temporal length. PhaseBreak
constructs a bijective temporal permutation $\pi\in\mathfrak{S}_{T_h}$.
Writing $\mathbf{i}=(0,\ldots,T_h-1)$ and
$\boldsymbol{\pi}=(\pi(0),\ldots,\pi(T_h-1))$, the schedule is constrained by
\begin{equation}
  \|\boldsymbol{\pi}-\mathbf{i}\|_{\infty}\leq r,\qquad
  \|\boldsymbol{\pi}-\mathbf{i}\|_{0}\leq B_0,\qquad
  \|\boldsymbol{\pi}-\mathbf{i}\|_{1}\leq B_1,
  \label{eq:retime-constraints}
\end{equation}
where $r$ is the retiming radius, $B_0$ bounds the number of moved temporal
positions, and $B_1$ bounds their total displacement. We generate a feasible
permutation as
$\boldsymbol{\pi}=\mathcal{G}_{\rho}(T_h;r,B_0,B_1,\zeta)$, where $\rho$ denotes
the schedule family and $\zeta$ fixes its stochastic realization. The generator
reorders existing values without duplicating or deleting a temporal position.
The retimed target is
\begin{equation}
  \widetilde{\mathbf{h}}^{c}
  = \mathcal{P}_{\pi}(\mathbf{h}^{c})
  = \bigl(\mathbf{h}^{c}_{\pi(q)}\bigr)_{q=0}^{T_h-1},
  \label{eq:retimed-target}
\end{equation}
and is detached from the computation graph.

\textbf{Dual-domain Retiming Objective.}
For a feature $\mathbf{h}$, define temporal centering and its orthonormal real
Fourier transform as
\begin{equation}
  \overline{\mathbf{h}}
  = \mathbf{h}-\mathbb{E}_{T_h}[\mathbf{h}],\qquad
  \mathbf{F}(\mathbf{h})
  = \operatorname{rFFT}_{T_h}(\overline{\mathbf{h}}).
  \label{eq:temporal-spectrum}
\end{equation}
The unit spectrum is
$\mathbf{U}(\mathbf{h})=\mathbf{F}(\mathbf{h})/(|\mathbf{F}(\mathbf{h})|+\eta)$,
where $\eta>0$ prevents division by zero. Excluding the zero-frequency bin,
the clean magnitude defines normalized weights
\begin{equation}
  \mathbf{w}_k
  = \frac{|\mathbf{F}_k(\mathbf{h}^{c})|}
  {\sum_{j>0}|\mathbf{F}_j(\mathbf{h}^{c})|+\eta},\qquad k>0.
  \label{eq:frequency-weights}
\end{equation}
The spectral branch matches the protected phase to the phase of the retimed
target:
\begin{equation}
  \mathcal{L}_{\mathrm{phase}}
  = \mathbb{E}\!\left[
      \sum_{k>0}\mathbf{w}_k
      \left|\mathbf{U}_k(\mathbf{h}^{p})
      -\mathbf{U}_k(\widetilde{\mathbf{h}}^{c})\right|^2
    \right].
  \label{eq:phase-loss}
\end{equation}
The time-domain branch directly aligns centered activations. With
$s=(\mathbb{E}[(\overline{\widetilde{\mathbf{h}}^{c}})^2])^{1/2}+\eta$, it is
\begin{equation}
  \mathcal{L}_{\mathrm{retime}}
  = \mathbb{E}\!\left[
      \left(
        \frac{\overline{\mathbf{h}}^{p}
        -\overline{\widetilde{\mathbf{h}}^{c}}}{s}
      \right)^2
    \right].
  \label{eq:retime-loss}
\end{equation}
These two terms form the retiming main line: Eq.~\eqref{eq:retime-loss}
specifies where temporal activations move, while Eq.~\eqref{eq:phase-loss}
matches the corresponding spectral timing structure.

\begin{table*}[t]
  \centering
  \small
  \setlength{\tabcolsep}{0.5pt}
  \renewcommand{\arraystretch}{0.88}
  \caption{\textbf{Main comparison with protection baselines.} Each protected
  cell reports the adversarial mean followed by its relative change from the
  shared clean mean in parentheses. Arrows indicate stronger disruption.}
  \label{tab:temporal-main}

  \resizebox{\textwidth}{!}{%
  \begin{tabular}{@{}c|l|ccccc@{}}
    \toprule
    \rowcolor{gray!20}\textbf{Model} & \textbf{Method} & \textbf{Acc.}$\uparrow$ & \textbf{Dir.}$\uparrow$ & \textbf{Jerk}$\uparrow$ & \textbf{Warp.}$\uparrow$ & \textbf{Smooth.}$\downarrow$ \\
    \midrule\midrule
    \multirow{7}{*}{\rotatebox[origin=c]{90}{SVD}} & Clean & \num{0.001462} & \num{0.019978} & \num{0.002270} & \num{0.022405} & \num{0.981167} \\
     & \graycell{ PhotoGuard} & \graycell{\result{0.001976}{+35.16}} & \graycell{\result{0.070912}{+254.95}} & \graycell{\result{0.002585}{+13.88}} & \graycell{\result{0.021762}{-2.87}} & \graycell{\result{0.981725}{+0.06}} \\
     & MFA & \result{0.001583}{+8.28} & \result{0.114267}{+471.96} & \result{0.002472}{+8.90} & \result{0.022020}{-1.72} & \result{0.976059}{-0.52} \\
     & \graycell{ I2VGuard} & \graycell{\result{0.001513}{+3.49}} & \graycell{\result{0.110851}{+454.87}} & \graycell{\result{0.003748}{+65.11}} & \graycell{\result{0.023346}{+4.20}} & \graycell{\result{0.987498}{+0.65}} \\
     & Vid-Freeze & \result{0.001483}{+1.44} & \result{0.102103}{+411.08} & \result{0.002418}{+6.52} & \result{0.023700}{+5.78} & \result{0.980561}{-0.06} \\
     & \graycell{ MEDUSA} & \graycell{\result{0.000250}{-82.90}} & \graycell{\textbf{\result{0.127252}{+536.96}}} & \graycell{\result{0.000424}{-81.32}} & \graycell{\result{0.009559}{-57.34}} & \graycell{\result{0.992043}{+1.11}} \\
    & \bluecell{\textbf{PhaseBreak (Ours)}} & \bluecell{\textbf{\result{0.002326}{+59.10}}} & \bluecell{\result{0.112167}{+461.45}} & \bluecell{\textbf{\result{0.003873}{+70.62}}} & \bluecell{\textbf{\result{0.029517}{+31.74}}} & \bluecell{\textbf{\result{0.975051}{-0.62}}} \\
    \midrule
    \multirow{7}{*}{\rotatebox[origin=c]{90}{DynamiCrafter}} & Clean & \num{0.003317} & \num{0.258702} & \num{0.005406} & \num{0.028229} & \num{0.977188} \\
     & \graycell{ PhotoGuard} & \graycell{\result{0.002638}{-20.47}} & \graycell{\result{0.33500}{+29.49}} & \graycell{\result{0.004948}{-8.47}} & \graycell{\result{0.029238}{+3.57}} & \graycell{\result{0.975864}{-0.14}} \\
     & MFA & \result{0.005687}{+71.45} & \result{0.198410}{-23.31} & \result{0.010354}{+91.53} & \result{0.041243}{+46.10} & \result{0.966840}{-1.06} \\
     & \graycell{ I2VGuard} & \graycell{\result{0.015390}{+363.97}} & \graycell{\result{0.328809}{+27.10}} & \graycell{\result{0.027944}{+416.91}} & \graycell{\result{0.048104}{+70.41}} & \graycell{\result{0.957872}{-1.98}} \\
     & Vid-Freeze & \result{0.002753}{-17.00} & \result{0.363420}{+40.48} & \result{0.005184}{-4.11} & \result{0.045276}{+60.39} & \result{0.964037}{-1.35} \\
     & \graycell{ MEDUSA} & \graycell{\result{0.023450}{+606.96}} & \graycell{\result{0.375700}{+45.23}} & \graycell{\result{0.041645}{+670.35}} & \graycell{\result{0.060005}{+112.57}} & \graycell{\result{0.94913}{-2.87}} \\
    & \bluecell{\textbf{PhaseBreak (Ours)}} & \bluecell{\textbf{\result{0.024396}{+635.48}}} & \bluecell{\textbf{\result{0.383502}{+48.24}}} & \bluecell{\textbf{\result{0.045559}{+742.75}}} & \bluecell{\textbf{\result{0.069542}{+146.35}}} & \bluecell{\textbf{\result{0.939679}{-3.84}}} \\
    \bottomrule
  \end{tabular}}
\end{table*}

\begin{figure*}[t]
  \centering
  \setlength{\tabcolsep}{2pt}
  \begin{tabular}{@{}cc@{}}
    \includegraphics[width=0.49\textwidth,viewport=85 405 760 505,clip]{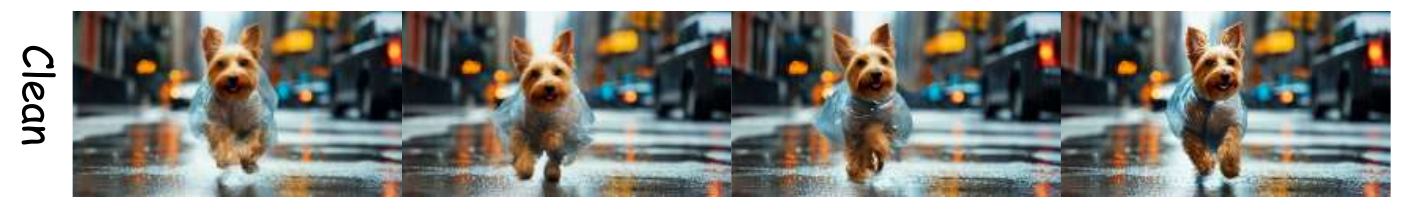} &
    \includegraphics[width=0.49\textwidth,viewport=85 143 760 243,clip]{phasebreak_qualitative.pdf} \\
    \includegraphics[width=0.49\textwidth,viewport=85 303 760 403,clip]{phasebreak_qualitative.pdf} &
    \includegraphics[width=0.49\textwidth,viewport=85 42 760 142,clip]{phasebreak_qualitative.pdf}
  \end{tabular}
  \caption{\textbf{Qualitative comparison of clean and protected generation.}
  Each column shows one reference example, with the clean sequence on top and
  the PhaseBreak sequence below.}
  \Description{Two qualitative examples arranged in two columns. Each column
  compares four clean generated frames on top with four PhaseBreak-protected
  frames below. The protected sequences exhibit increasing deformation of the
  animated animal subject.}
  \label{fig:qualitative-results}
\end{figure*}

\textbf{Auxiliary Dynamics Branches.}
Four differentiable terms supplement, but do not redefine, the retiming
target. Let
$\Delta\mathbf{h}_q=\mathbf{h}_{q+1}-\mathbf{h}_q$. The twitch term minimizes
the cosine similarity of consecutive differences,
\begin{equation}
  \mathcal{L}_{\mathrm{twitch}}
  = \mathbb{E}_{q}\!\left[
      \frac{\langle\Delta\mathbf{h}^{p}_q,
      \Delta\mathbf{h}^{p}_{q+1}\rangle}
      {\|\Delta\mathbf{h}^{p}_q\|_2
       \|\Delta\mathbf{h}^{p}_{q+1}\|_2+\eta}
    \right],
  \label{eq:twitch-loss}
\end{equation}
favoring alternating temporal directions. The high-frequency term is the
negative fraction of nonzero spectral energy at frequencies at or above a
cutoff $k_c$:
\begin{equation}
  \mathcal{L}_{\mathrm{high}}
  = -\mathbb{E}\!\left[
      \frac{\sum_{k\geq k_c}|\mathbf{F}_k(\mathbf{h}^{p})|^2}
      {\sum_{k>0}|\mathbf{F}_k(\mathbf{h}^{p})|^2+\eta}
    \right].
  \label{eq:high-frequency-loss}
\end{equation}
The motion and acceleration branches respectively maximize normalized first-
and second-difference energy:
\begin{equation}
  \resizebox{0.98\columnwidth}{!}{$\displaystyle
  \mathcal{L}_{\mathrm{mot+}}=-\frac{\mathbb{E}[\|\Delta\mathbf{h}^{p}\|_2^2]}
  {\mathbb{E}[\|\overline{\mathbf{h}}^{p}\|_2^2]+\eta},\quad
  \mathcal{L}_{\mathrm{acc+}}=-\frac{\mathbb{E}[\|\Delta^2\mathbf{h}^{p}\|_2^2]}
  {\mathbb{E}[\|\Delta\mathbf{h}^{p}\|_2^2]+\eta}. $}
  \label{eq:dynamics-boost-losses}
\end{equation}

\subsection{Masked Projected Optimization}
\textbf{Joint Objective and Update.}
For each selected denoiser layer $\ell\in\mathcal{L}$ and diffusion timestep
$t\in\mathcal{T}$, PhaseBreak evaluates the active objective
\begin{equation}
  \begin{aligned}
  \mathcal{J}_{\ell,t}
  ={}& \lambda_{\mathrm{phase}}\mathcal{L}_{\mathrm{phase}}
  + \lambda_{\mathrm{retime}}\mathcal{L}_{\mathrm{retime}}
  + \lambda_{\mathrm{twitch}}\mathcal{L}_{\mathrm{twitch}}
  + \lambda_{\mathrm{high}}\mathcal{L}_{\mathrm{high}}\\
  &+ \lambda_{\mathrm{mot+}}\mathcal{L}_{\mathrm{mot+}}
  + \lambda_{\mathrm{acc+}}\mathcal{L}_{\mathrm{acc+}},
  \end{aligned}
  \label{eq:complete-objective}
\end{equation}
PhaseBreak averages Eq.~\eqref{eq:complete-objective} across
$\mathcal{L}\times\mathcal{T}$, as in Eq.~\eqref{eq:phasebreak-objective}.
Clean features are computed once and detached. Following
PGD~\cite{madry2018towards}, the perturbation update at iteration $k$ is
\begin{equation}
  \boldsymbol{\delta}_{k+1}
  = \Pi_{\mathcal{S}}\!\left(
      \boldsymbol{\delta}_{k}
      -\alpha\,\operatorname{sign}\!\left(
        \mathbf{M}\odot\nabla_{\boldsymbol{\delta}_{k}}\mathcal{J}
      \right)
    \right),
  \label{eq:masked-pgd-update}
\end{equation}
where $\alpha$ is the step size. The clean and protected branches share the
prompt, latent noise, and evaluated diffusion steps. Only the protected
branch contributes gradients.

\section{Experiments}
\subsection{Experimental Setup}
\textbf{Models and Data.}
We evaluate Stable Video Diffusion (SVD)~\cite{blattmann2023stablevideo} and
DynamiCrafter~\cite{xing2023dynamicrafter} on 355 examples from the VBench-I2V
collection~\cite{huang2024vbench,huang2026vbenchpp}. SVD uses the $16{:}9$ crops and subject masks at a resolution of
$1024\times576$ and generates 14 frames. DynamiCrafter uses the $1{:}1$ crops
and masks at $256\times256$ and generates 16 frames. DynamiCrafter receives the
English prompt associated with each example, whereas SVD uses image
conditioning without a text prompt.

\textbf{Protection Configuration.}
We optimize each reference image for 250 masked PGD iterations using an
$\ell_\infty$ budget of $16/255$, a step size of $0.01$, and zero
initialization. Clean and protected branches share the prompt, initial noise,
and diffusion timestep ($t=800$) to isolate the effect of the input
perturbation. For each frozen denoiser, we extract one intermediate temporal
representation and apply PhaseBreak along its verified temporal axis, which
is dimension 1 with 14 positions for SVD and dimension 2 with 16 positions for
DynamiCrafter. We instantiate $\mathcal{G}_{\rho}$ using non-overlapping
adjacent swaps with $r=1$ and a fixed seed. The movement budgets equal the
corresponding temporal length, allowing each position to move by at most one
step.

\textbf{Loss Weights.}
We set $\lambda_{\mathrm{phase}}$ to $0.05$,
$\lambda_{\mathrm{retime}}$ to $0.2$, $\lambda_{\mathrm{twitch}}$ to $3.0$,
$\lambda_{\mathrm{high}}$ to $2.0$, and both
$\lambda_{\mathrm{mot+}}$ and $\lambda_{\mathrm{acc+}}$ to $0.8$. SVD uses
full-precision optimization, whereas DynamiCrafter uses mixed precision.

\subsection{Evaluation Metrics}
We report five optical-flow-based temporal measurements. Flow acceleration and
flow jerk measure first- and higher-order changes in motion magnitude; larger
values indicate stronger short-term variation. Flow direction inconsistency
measures changes in motion orientation, and flow warping error measures
frame-to-frame alignment error; both increase as temporal consistency degrades.
Motion smoothness measures the regularity of the motion trajectory, for which a
lower value indicates stronger disruption. We report the clean mean, the
protected mean, and their relative change under matched generation settings.

\subsection{Main Results}
\textbf{Cross-Model Temporal Disruption.}
We compare PhaseBreak with PhotoGuard~\cite{salman2023photoguard},
MFA~\cite{yu2024mfa},
I2VGuard~\cite{gui2025i2vguard}, Vid-Freeze~\cite{chowdhury2025vidfreeze}, and
MEDUSA~\cite{yu2026medusa} to test consistent disruption of motion dynamics.
Table~\ref{tab:temporal-main} shows that PhaseBreak moves all
five metrics in the effective direction and provides the strongest overall
disruption across the two evaluated models.

\subsection{Transferability}
\textbf{Bidirectional Cross-Model Transfer.}
We transfer protected references between the two architectures to assess
cross-model effectiveness under matched generation settings.
Table~\ref{tab:transferability} shows that all five metrics move in the
effective direction in both transfer settings. These results demonstrate that
the learned protection generalizes bidirectionally across I2V architectures.

\begin{table}[t]
  \centering
  \small
  \setlength{\tabcolsep}{1.2pt}
  \renewcommand{\arraystretch}{1.0}
  \caption{\textbf{Bidirectional transferability.} Relative changes when
  protections optimized on the training model are evaluated with the inference
  model. DC denotes DynamiCrafter.}
  \label{tab:transferability}
  \resizebox{\columnwidth}{!}{%
  \begin{tabular}{@{}cc|ccccc@{}}
    \toprule
    \rowcolor{gray!20}
    \textbf{Training} & \textbf{Inference} & \textbf{Acc.}$\uparrow$
    & \textbf{Dir.}$\uparrow$ & \textbf{Jerk}$\uparrow$
    & \textbf{Warp.}$\uparrow$ & \textbf{Smooth.}$\downarrow$ \\
    \midrule\midrule
    \graycell{SVD} & \graycell{DC} & \graycell{\percent{+28.75}}
      & \graycell{\percent{+2.29}} & \graycell{\percent{+36.50}}
      & \graycell{\percent{+31.92}} & \graycell{\percent{-0.97}} \\
    DC & SVD & \percent{+143.71} & \percent{+237.43} & \percent{+159.23}
      & \percent{+30.62} & \percent{-0.44} \\
    \bottomrule
  \end{tabular}}
\end{table}

\subsection{Perturbation Budget}
\textbf{Effect of the Perturbation Budget.}
We evaluate BudgetSwap at $\varepsilon\in\{4,8,12\}/255$ to examine sensitivity to
the perturbation budget. Table~\ref{tab:budget-swap} shows that PhaseBreak
consistently disrupts temporal dynamics across all evaluated budgets, while
the larger budget produces particularly strong effects on DynamiCrafter.

\begin{table}[t]
  \centering
  \small
  \setlength{\tabcolsep}{0.5pt}
  \caption{\textbf{BudgetSwap.} Relative changes under three $\ell_\infty$
  perturbation budgets. DC denotes DynamiCrafter.}
  \label{tab:budget-swap}

  \begin{tabular}{@{}>{\centering\arraybackslash}p{0.125\columnwidth}|
    >{\raggedright\arraybackslash}p{0.18\columnwidth}|
    *{5}{>{\centering\arraybackslash}p{0.13\columnwidth}}@{}}
    \toprule
    \rowcolor{gray!20}\textbf{Model} & \textbf{$\varepsilon$} & \textbf{Acc.}$\uparrow$ & \textbf{Dir.}$\uparrow$ & \textbf{Jerk}$\uparrow$ & \textbf{Warp.}$\uparrow$ & \textbf{Smth.}$\downarrow$ \\
    \midrule\midrule
    \multirow{3}{*}{\rotatebox[origin=c]{90}{SVD}} & \graycell{$4/255$} & \graycell{\textbf{\percent{+38.72}}} & \graycell{\percent{+209.31}} & \graycell{\textbf{\percent{+36.80}}} & \graycell{\percent{+12.39}} & \graycell{\percent{-0.31}} \\
     & $8/255$ & \percent{+14.48} & \percent{+304.27} & \percent{+15.65} & \percent{+17.60} & \percent{-0.39} \\
    & \graycell{$12/255$} & \graycell{\percent{+20.12}} & \graycell{\textbf{\percent{+480.10}}} & \graycell{\percent{+22.17}} & \graycell{\textbf{\percent{+20.48}}} & \graycell{\textbf{\percent{-0.43}}} \\
    \midrule
    \multirow{3}{*}{\rotatebox[origin=c]{90}{DC}} & \graycell{$4/255$} & \graycell{\percent{+461.30}} & \graycell{\percent{+31.83}} & \graycell{\percent{+546.24}} & \graycell{\percent{+88.76}} & \graycell{\percent{-2.55}} \\
     & $8/255$ & \percent{+286.52} & \percent{+41.51} & \percent{+343.62} & \percent{+83.84} & \percent{-2.41} \\
    & \graycell{$12/255$} & \graycell{\textbf{\percent{+547.85}}} & \graycell{\textbf{\percent{+51.65}}} & \graycell{\textbf{\percent{+615.22}}} & \graycell{\textbf{\percent{+108.56}}} & \graycell{\textbf{\percent{-3.21}}} \\
    \bottomrule
  \end{tabular}
\end{table}

\subsection{Ablation Study}
\textbf{Contribution of Loss Components.}
We remove each evaluated term while retaining
$\mathcal{L}_{\mathrm{phase}}$ and all other terms in
Eq.~\eqref{eq:complete-objective}. Table~\ref{tab:loss-ablation} shows that the
full objective provides the strongest and most balanced temporal disruption
across the two architectures. Together, the ablation results confirm that
retiming and the auxiliary dynamics terms provide complementary temporal
supervision.

\begin{table}[t]
  \centering
  \small
  \setlength{\tabcolsep}{0.5pt}
  \caption{\textbf{Loss ablation.} Relative changes after removing one
  evaluated PhaseBreak objective term; $\mathcal{L}_{\mathrm{phase}}$ remains
  active in every variant.}
  \label{tab:loss-ablation}

  \begin{tabular}{@{}>{\centering\arraybackslash}p{0.125\columnwidth}|
    >{\raggedright\arraybackslash}p{0.20\columnwidth}|
    *{5}{>{\centering\arraybackslash}p{0.126\columnwidth}}@{}}
    \toprule
    \rowcolor{gray!20}\textbf{Model} & \textbf{Variant} & \textbf{Acc.}$\uparrow$ & \textbf{Dir.}$\uparrow$ & \textbf{Jerk}$\uparrow$ & \textbf{Warp.}$\uparrow$ & \textbf{Smth.}$\downarrow$ \\
    \midrule\midrule
    \multirow{6}{*}{\rotatebox[origin=c]{90}{SVD}} & \graycell{w/o $\mathcal{L}_{\mathrm{retime}}$} & \graycell{\percent{+42.74}} & \graycell{\percent{+421.28}} & \graycell{\percent{+41.68}} & \graycell{\percent{+23.29}} & \graycell{\percent{-0.57}} \\
     & w/o $\mathcal{L}_{\mathrm{twitch}}$ & \percent{+53.84} & \percent{+452.70} & \percent{+57.22} & \percent{+22.93} & \percent{-0.52} \\
     & \graycell{w/o $\mathcal{L}_{\mathrm{high}}$} & \graycell{\percent{+54.70}} & \graycell{\percent{+443.12}} & \graycell{\percent{+66.56}} & \graycell{\percent{+24.00}} & \graycell{\percent{-0.54}} \\
     & w/o $\mathcal{L}_{\mathrm{mot+}}$ & \percent{+56.22} & \percent{+391.63} & \percent{+65.40} & \percent{+27.61} & \percent{-0.58} \\
     & \graycell{w/o $\mathcal{L}_{\mathrm{acc+}}$} & \graycell{\percent{+25.80}} & \graycell{\percent{+450.71}} & \graycell{\percent{+25.00}} & \graycell{\percent{+29.64}} & \graycell{\percent{-0.59}} \\
    & \bluecell{\textbf{Full (Ours)}} & \bluecell{\textbf{\percent{+59.10}}} & \bluecell{\textbf{\percent{+461.45}}} & \bluecell{\textbf{\percent{+70.62}}} & \bluecell{\textbf{\percent{+31.74}}} & \bluecell{\textbf{\percent{-0.62}}} \\
    \midrule
    \multirow{6}{*}{\rotatebox[origin=c]{90}{DynamiCrafter}} & \graycell{w/o $\mathcal{L}_{\mathrm{retime}}$} & \graycell{\percent{+548.07}} & \graycell{\percent{+48.07}} & \graycell{\percent{+717.87}} & \graycell{\percent{+141.81}} & \graycell{\percent{-3.42}} \\
     & w/o $\mathcal{L}_{\mathrm{twitch}}$ & \percent{+579.27} & \percent{+43.89} & \percent{+673.76} & \textbf{\percent{+148.95}} & \percent{-3.52} \\
     & \graycell{w/o $\mathcal{L}_{\mathrm{high}}$} & \graycell{\percent{+422.94}} & \graycell{\percent{+46.35}} & \graycell{\percent{+479.53}} & \graycell{\percent{+115.31}} & \graycell{\percent{-3.26}} \\
     & w/o $\mathcal{L}_{\mathrm{mot+}}$ & \percent{+339.09} & \percent{+45.35} & \percent{+375.27} & \percent{+93.45} & \percent{-2.66} \\
     & \graycell{w/o $\mathcal{L}_{\mathrm{acc+}}$} & \graycell{\percent{+460.98}} & \graycell{\percent{+46.96}} & \graycell{\percent{+524.57}} & \graycell{\percent{+141.30}} & \graycell{\percent{-3.74}} \\
    & \bluecell{\textbf{Full (Ours)}} & \bluecell{\textbf{\percent{+635.48}}} & \bluecell{\textbf{\percent{+48.24}}} & \bluecell{\textbf{\percent{+742.75}}} & \bluecell{\percent{+146.35}} & \bluecell{\textbf{\percent{-3.84}}} \\
    \bottomrule
  \end{tabular}
\end{table}

\subsection{Additional Experiments}
Loss-weight sensitivity experiments are provided in
Appendix~\ref{app:loss-weight-sensitivity}. Evaluations of the visual quality
of protected input images are reported in Appendix~\ref{app:input-image-quality}.

\section{Limitations}
We evaluate short image-conditioned videos of 14 or 16 frames from VBench-I2V.
This setting does not establish whether feature retiming remains effective in
long-horizon generation. PhaseBreak assumes white-box access to intermediate
denoiser features and gradients of the anticipated I2V generator, restricting
direct application to accessible models rather than API-only services.

Protection is optimized offline for each image before release. This workflow
excludes already-distributed images and applications requiring immediate
protection without per-image optimization.

\section{Conclusion}
PhaseBreak proactively protects shared images against unauthorized I2V
animation by matching intermediate denoiser features to constrained temporal
rearrangements of clean features, while a subject-aware mask confines
perturbations to the depicted subject. This mechanism targets temporal ordering
rather than globally attenuating features. Experiments on Stable Video Diffusion
and DynamiCrafter show increased short-term motion variation, reduced temporal
regularity, and effective bidirectional cross-model transfer, establishing
temporal retiming as an effective mechanism for preventing stable motion
customization from released images.

\clearpage
\bibliographystyle{ACM-Reference-Format}
\bibliography{ref}


\begin{thebibliography}{36}


\ifx \showCODEN    \undefined \def \showCODEN     #1{\unskip}     \fi
\ifx \showISBNx    \undefined \def \showISBNx     #1{\unskip}     \fi
\ifx \showISBNxiii \undefined \def \showISBNxiii  #1{\unskip}     \fi
\ifx \showISSN     \undefined \def \showISSN      #1{\unskip}     \fi
\ifx \showLCCN     \undefined \def \showLCCN      #1{\unskip}     \fi
\ifx \shownote     \undefined \def \shownote      #1{#1}          \fi
\ifx \showarticletitle \undefined \def \showarticletitle #1{#1}   \fi
\ifx \showURL      \undefined \def \showURL       {\relax}        \fi
\providecommand\bibfield[2]{#2}
\providecommand\bibinfo[2]{#2}
\providecommand\natexlab[1]{#1}
\providecommand\showeprint[2][]{arXiv:#2}
\makeatletter
\@ifundefined{NAT@parse@date}{}{\let\NAT@parse@date@orig\NAT@parse@date}
\@ifundefined{NAT@parse@date}{}{\def\NAT@parse@date#1#2#3#4#5#6@@{\NAT@parse@date@orig#1#2#3#4#5#6@@\def\NAT@tempyear{0000}\def\NAT@tempexlab{{?}}\ifx\NAT@year\NAT@tempyear\ifx\NAT@exlab\NAT@tempexlab\def\NAT@date{[n.\,d.]}\else\edef\NAT@date{[n.\,d.]\NAT@exlab}\fi\fi}}
\makeatother

\bibitem[Blattmann et~al\mbox{.}(2023)]%
        {blattmann2023stablevideo}
\bibfield{author}{\bibinfo{person}{Andreas Blattmann}, \bibinfo{person}{Tim
  Dockhorn}, \bibinfo{person}{Sumith Kulal}, \bibinfo{person}{Daniel
  Mendelevitch}, \bibinfo{person}{Maciej Kilian}, \bibinfo{person}{Dominik
  Lorenz}, \bibinfo{person}{Yam Levi}, \bibinfo{person}{Zion English},
  \bibinfo{person}{Vikram Voleti}, \bibinfo{person}{Adam Letts},
  \bibinfo{person}{Varun Jampani}, {and} \bibinfo{person}{Robin Rombach}.}
  \bibinfo{year}{2023}\natexlab{}.
\newblock \showarticletitle{{Stable Video Diffusion}: Scaling Latent Video
  Diffusion Models to Large Datasets}.
\newblock \bibinfo{journal}{\emph{arXiv preprint arXiv:2311.15127}}
  (\bibinfo{year}{2023}).
\newblock


\bibitem[Chen et~al\mbox{.}(2023)]%
        {chen2023videocrafter1}
\bibfield{author}{\bibinfo{person}{Haoxin Chen}, \bibinfo{person}{Menghan Xia},
  \bibinfo{person}{Yingqing He}, \bibinfo{person}{Yong Zhang},
  \bibinfo{person}{Xiaodong Cun}, \bibinfo{person}{Shaoshu Yang},
  \bibinfo{person}{Jinbo Xing}, \bibinfo{person}{Yaofang Liu},
  \bibinfo{person}{Qifeng Chen}, \bibinfo{person}{Xintao Wang},
  \bibinfo{person}{Chao Weng}, {and} \bibinfo{person}{Ying Shan}.}
  \bibinfo{year}{2023}\natexlab{}.
\newblock \showarticletitle{{VideoCrafter1}: Open Diffusion Models for
  High-Quality Video Generation}.
\newblock \bibinfo{journal}{\emph{arXiv preprint arXiv:2310.19512}}
  (\bibinfo{year}{2023}).
\newblock


\bibitem[Chowdhury et~al\mbox{.}(2025)]%
        {chowdhury2025vidfreeze}
\bibfield{author}{\bibinfo{person}{Rohit Chowdhury}, \bibinfo{person}{Aniruddha
  Bala}, \bibinfo{person}{Rohan Jaiswal}, {and} \bibinfo{person}{Siddharth
  Roheda}.} \bibinfo{year}{2025}\natexlab{}.
\newblock \showarticletitle{{Vid-Freeze}: Protecting Images from Malicious
  Image-to-Video Generation via Temporal Freezing}.
\newblock \bibinfo{journal}{\emph{arXiv preprint arXiv:2509.23279}}
  (\bibinfo{year}{2025}).
\newblock


\bibitem[Dai et~al\mbox{.}(2024)]%
        {dai2024advdiff}
\bibfield{author}{\bibinfo{person}{Xuelong Dai}, \bibinfo{person}{Kaisheng
  Liang}, {and} \bibinfo{person}{Bin Xiao}.} \bibinfo{year}{2024}\natexlab{}.
\newblock \showarticletitle{{AdvDiff}: Generating Unrestricted Adversarial
  Examples using Diffusion Models}. In \bibinfo{booktitle}{\emph{Computer
  Vision--ECCV 2024}}, Vol.~\bibinfo{volume}{15104}.
  \bibinfo{publisher}{Springer}, \bibinfo{pages}{93--109}.
\newblock


\bibitem[Goodfellow et~al\mbox{.}(2015)]%
        {goodfellow2015explaining}
\bibfield{author}{\bibinfo{person}{Ian~J. Goodfellow},
  \bibinfo{person}{Jonathon Shlens}, {and} \bibinfo{person}{Christian
  Szegedy}.} \bibinfo{year}{2015}\natexlab{}.
\newblock \showarticletitle{Explaining and Harnessing Adversarial Examples}. In
  \bibinfo{booktitle}{\emph{International Conference on Learning
  Representations}}.
\newblock


\bibitem[Gui et~al\mbox{.}(2025)]%
        {gui2025i2vguard}
\bibfield{author}{\bibinfo{person}{Dongnan Gui}, \bibinfo{person}{Xun Guo},
  \bibinfo{person}{Wengang Zhou}, {and} \bibinfo{person}{Yan Lu}.}
  \bibinfo{year}{2025}\natexlab{}.
\newblock \showarticletitle{{I2VGuard}: Safeguarding Images against Misuse in
  Diffusion-based Image-to-Video Models}. In
  \bibinfo{booktitle}{\emph{Proceedings of the IEEE/CVF Conference on Computer
  Vision and Pattern Recognition}}. \bibinfo{pages}{12595--12604}.
\newblock


\bibitem[Guo et~al\mbox{.}(2024)]%
        {guo2023animatediff}
\bibfield{author}{\bibinfo{person}{Yuwei Guo}, \bibinfo{person}{Ceyuan Yang},
  \bibinfo{person}{Anyi Rao}, \bibinfo{person}{Zhengyang Liang},
  \bibinfo{person}{Yaohui Wang}, \bibinfo{person}{Yu Qiao},
  \bibinfo{person}{Maneesh Agrawala}, \bibinfo{person}{Dahua Lin}, {and}
  \bibinfo{person}{Bo Dai}.} \bibinfo{year}{2024}\natexlab{}.
\newblock \showarticletitle{{AnimateDiff}: Animate Your Personalized
  Text-to-Image Diffusion Models without Specific Tuning}. In
  \bibinfo{booktitle}{\emph{International Conference on Learning
  Representations}}.
\newblock


\bibitem[HaCohen et~al\mbox{.}(2024)]%
        {hacohen2024ltxvideo}
\bibfield{author}{\bibinfo{person}{Yoav HaCohen}, \bibinfo{person}{Nisan
  Chiprut}, \bibinfo{person}{Benny Brazowski}, \bibinfo{person}{Daniel Shalem},
  \bibinfo{person}{Dudu Moshe}, \bibinfo{person}{Eitan Richardson},
  \bibinfo{person}{Eran Levin}, \bibinfo{person}{Guy Shiran},
  \bibinfo{person}{Nir Zabari}, \bibinfo{person}{Ori Gordon},
  \bibinfo{person}{Poriya Panet}, \bibinfo{person}{Sapir Weissbuch},
  \bibinfo{person}{Victor Kulikov}, \bibinfo{person}{Yaki Bitterman},
  \bibinfo{person}{Zeev Melumian}, {and} \bibinfo{person}{Ofir Bibi}.}
  \bibinfo{year}{2024}\natexlab{}.
\newblock \showarticletitle{{LTX-Video}: Realtime Video Latent Diffusion}.
\newblock \bibinfo{journal}{\emph{arXiv preprint arXiv:2501.00103}}
  (\bibinfo{year}{2024}).
\newblock


\bibitem[He et~al\mbox{.}(2026)]%
        {he2026tear}
\bibfield{author}{\bibinfo{person}{Jiaming He}, \bibinfo{person}{Guanyu Hou},
  \bibinfo{person}{Hongwei Li}, \bibinfo{person}{Zhicong Huang},
  \bibinfo{person}{Kangjie Chen}, \bibinfo{person}{Yi Yu},
  \bibinfo{person}{Wenbo Jiang}, \bibinfo{person}{Guowen Xu}, {and}
  \bibinfo{person}{Tianwei Zhang}.} \bibinfo{year}{2026}\natexlab{}.
\newblock \showarticletitle{{TEAR}: Temporal-aware Automated Red-teaming for
  Text-to-Video Models}. In \bibinfo{booktitle}{\emph{Proceedings of the
  IEEE/CVF Conference on Computer Vision and Pattern Recognition}}.
  \bibinfo{pages}{41--50}.
\newblock


\bibitem[Ho et~al\mbox{.}(2022a)]%
        {ho2022imagenvideo}
\bibfield{author}{\bibinfo{person}{Jonathan Ho}, \bibinfo{person}{William
  Chan}, \bibinfo{person}{Chitwan Saharia}, \bibinfo{person}{Jay Whang},
  \bibinfo{person}{Ruiqi Gao}, \bibinfo{person}{Alexey Gritsenko},
  \bibinfo{person}{Diederik~P. Kingma}, \bibinfo{person}{Ben Poole},
  \bibinfo{person}{Mohammad Norouzi}, \bibinfo{person}{David~J. Fleet}, {and}
  \bibinfo{person}{Tim Salimans}.} \bibinfo{year}{2022}\natexlab{a}.
\newblock \showarticletitle{{Imagen Video}: High Definition Video Generation
  with Diffusion Models}.
\newblock \bibinfo{journal}{\emph{arXiv preprint arXiv:2210.02303}}
  (\bibinfo{year}{2022}).
\newblock


\bibitem[Ho et~al\mbox{.}(2020)]%
        {ho2020ddpm}
\bibfield{author}{\bibinfo{person}{Jonathan Ho}, \bibinfo{person}{Ajay Jain},
  {and} \bibinfo{person}{Pieter Abbeel}.} \bibinfo{year}{2020}\natexlab{}.
\newblock \showarticletitle{Denoising Diffusion Probabilistic Models}. In
  \bibinfo{booktitle}{\emph{Advances in Neural Information Processing
  Systems}}, Vol.~\bibinfo{volume}{33}. \bibinfo{pages}{6840--6851}.
\newblock


\bibitem[Ho et~al\mbox{.}(2022b)]%
        {ho2022videodiffusion}
\bibfield{author}{\bibinfo{person}{Jonathan Ho}, \bibinfo{person}{Tim
  Salimans}, \bibinfo{person}{Alexey Gritsenko}, \bibinfo{person}{William
  Chan}, \bibinfo{person}{Mohammad Norouzi}, {and} \bibinfo{person}{David~J.
  Fleet}.} \bibinfo{year}{2022}\natexlab{b}.
\newblock \showarticletitle{Video Diffusion Models}. In
  \bibinfo{booktitle}{\emph{Advances in Neural Information Processing
  Systems}}, Vol.~\bibinfo{volume}{35}. \bibinfo{pages}{8633--8646}.
\newblock


\bibitem[Huang et~al\mbox{.}(2024)]%
        {huang2024vbench}
\bibfield{author}{\bibinfo{person}{Ziqi Huang}, \bibinfo{person}{Yinan He},
  \bibinfo{person}{Jiashuo Yu}, \bibinfo{person}{Fan Zhang},
  \bibinfo{person}{Chenyang Si}, \bibinfo{person}{Yuming Jiang},
  \bibinfo{person}{Yuanhan Zhang}, \bibinfo{person}{Tianxing Wu},
  \bibinfo{person}{Qingyang Jin}, \bibinfo{person}{Nattapol Chanpaisit},
  \bibinfo{person}{Yaohui Wang}, \bibinfo{person}{Xinyuan Chen},
  \bibinfo{person}{Limin Wang}, \bibinfo{person}{Dahua Lin},
  \bibinfo{person}{Yu Qiao}, {and} \bibinfo{person}{Ziwei Liu}.}
  \bibinfo{year}{2024}\natexlab{}.
\newblock \showarticletitle{{VBench}: Comprehensive Benchmark Suite for Video
  Generative Models}. In \bibinfo{booktitle}{\emph{Proceedings of the IEEE/CVF
  Conference on Computer Vision and Pattern Recognition}}.
  \bibinfo{pages}{21807--21818}.
\newblock


\bibitem[Huang et~al\mbox{.}(2026)]%
        {huang2026vbenchpp}
\bibfield{author}{\bibinfo{person}{Ziqi Huang}, \bibinfo{person}{Fan Zhang},
  \bibinfo{person}{Xiaojie Xu}, \bibinfo{person}{Yinan He},
  \bibinfo{person}{Jiashuo Yu}, \bibinfo{person}{Ziyue Dong},
  \bibinfo{person}{Qianli Ma}, \bibinfo{person}{Nattapol Chanpaisit},
  \bibinfo{person}{Chenyang Si}, \bibinfo{person}{Yuming Jiang},
  \bibinfo{person}{Yaohui Wang}, \bibinfo{person}{Xinyuan Chen},
  \bibinfo{person}{Ying-Cong Chen}, \bibinfo{person}{Limin Wang},
  \bibinfo{person}{Dahua Lin}, \bibinfo{person}{Yu Qiao}, {and}
  \bibinfo{person}{Ziwei Liu}.} \bibinfo{year}{2026}\natexlab{}.
\newblock \showarticletitle{{VBench++}: Comprehensive and Versatile Benchmark
  Suite for Video Generative Models}.
\newblock \bibinfo{journal}{\emph{IEEE Transactions on Pattern Analysis and
  Machine Intelligence}} \bibinfo{volume}{48}, \bibinfo{number}{3}
  (\bibinfo{year}{2026}), \bibinfo{pages}{3268--3285}.
\newblock


\bibitem[Kang et~al\mbox{.}(2023)]%
        {kang2023diffattack}
\bibfield{author}{\bibinfo{person}{Mintong Kang}, \bibinfo{person}{Dawn Song},
  {and} \bibinfo{person}{Bo Li}.} \bibinfo{year}{2023}\natexlab{}.
\newblock \showarticletitle{{DiffAttack}: Evasion Attacks Against
  Diffusion-Based Adversarial Purification}. In
  \bibinfo{booktitle}{\emph{Advances in Neural Information Processing
  Systems}}, Vol.~\bibinfo{volume}{36}.
\newblock


\bibitem[Li et~al\mbox{.}(2025)]%
        {li2025t2vattack}
\bibfield{author}{\bibinfo{person}{Changzhen Li}, \bibinfo{person}{Yuecong
  Min}, \bibinfo{person}{Jie Zhang}, \bibinfo{person}{Zheng Yuan},
  \bibinfo{person}{Shiguang Shan}, {and} \bibinfo{person}{Xilin Chen}.}
  \bibinfo{year}{2025}\natexlab{}.
\newblock \showarticletitle{{T2VAttack}: Adversarial Attack on Text-to-Video
  Diffusion Models}.
\newblock \bibinfo{journal}{\emph{arXiv preprint arXiv:2512.23953}}
  (\bibinfo{year}{2025}).
\newblock


\bibitem[Liang and Wu(2023)]%
        {liang2023mist}
\bibfield{author}{\bibinfo{person}{Chumeng Liang} {and} \bibinfo{person}{Xiaoyu
  Wu}.} \bibinfo{year}{2023}\natexlab{}.
\newblock \showarticletitle{Mist: Towards Improved Adversarial Examples for
  Diffusion Models}.
\newblock \bibinfo{journal}{\emph{arXiv preprint arXiv:2305.12683}}
  (\bibinfo{year}{2023}).
\newblock


\bibitem[Madry et~al\mbox{.}(2018)]%
        {madry2018towards}
\bibfield{author}{\bibinfo{person}{Aleksander Madry},
  \bibinfo{person}{Aleksandar Makelov}, \bibinfo{person}{Ludwig Schmidt},
  \bibinfo{person}{Dimitris Tsipras}, {and} \bibinfo{person}{Adrian Vladu}.}
  \bibinfo{year}{2018}\natexlab{}.
\newblock \showarticletitle{Towards Deep Learning Models Resistant to
  Adversarial Attacks}. In \bibinfo{booktitle}{\emph{International Conference
  on Learning Representations}}.
\newblock


\bibitem[Rombach et~al\mbox{.}(2022)]%
        {rombach2022latent}
\bibfield{author}{\bibinfo{person}{Robin Rombach}, \bibinfo{person}{Andreas
  Blattmann}, \bibinfo{person}{Dominik Lorenz}, \bibinfo{person}{Patrick
  Esser}, {and} \bibinfo{person}{Bj{\"o}rn Ommer}.}
  \bibinfo{year}{2022}\natexlab{}.
\newblock \showarticletitle{High-Resolution Image Synthesis with Latent
  Diffusion Models}. In \bibinfo{booktitle}{\emph{Proceedings of the IEEE/CVF
  Conference on Computer Vision and Pattern Recognition}}.
  \bibinfo{pages}{10684--10695}.
\newblock


\bibitem[Salman et~al\mbox{.}(2023)]%
        {salman2023photoguard}
\bibfield{author}{\bibinfo{person}{Hadi Salman}, \bibinfo{person}{Alaa
  Khaddaj}, \bibinfo{person}{Guillaume Leclerc}, \bibinfo{person}{Andrew
  Ilyas}, {and} \bibinfo{person}{Aleksander Madry}.}
  \bibinfo{year}{2023}\natexlab{}.
\newblock \showarticletitle{Raising the Cost of Malicious {AI}-Powered Image
  Editing}. In \bibinfo{booktitle}{\emph{Proceedings of the 40th International
  Conference on Machine Learning}} \emph{(\bibinfo{series}{Proceedings of
  Machine Learning Research}, Vol.~\bibinfo{volume}{202})}.
  \bibinfo{pages}{29894--29918}.
\newblock


\bibitem[Singer et~al\mbox{.}(2022)]%
        {singer2022makeavideo}
\bibfield{author}{\bibinfo{person}{Uriel Singer}, \bibinfo{person}{Adam
  Polyak}, \bibinfo{person}{Thomas Hayes}, \bibinfo{person}{Xi Yin},
  \bibinfo{person}{Jie An}, \bibinfo{person}{Songyang Zhang},
  \bibinfo{person}{Qiyuan Hu}, \bibinfo{person}{Harry Yang},
  \bibinfo{person}{Oron Ashual}, \bibinfo{person}{Oran Gafni},
  \bibinfo{person}{Devi Parikh}, \bibinfo{person}{Sonal Gupta}, {and}
  \bibinfo{person}{Yaniv Taigman}.} \bibinfo{year}{2022}\natexlab{}.
\newblock \showarticletitle{{Make-A-Video}: Text-to-Video Generation without
  Text-Video Data}.
\newblock \bibinfo{journal}{\emph{arXiv preprint arXiv:2209.14792}}
  (\bibinfo{year}{2022}).
\newblock


\bibitem[Song et~al\mbox{.}(2021)]%
        {song2021ddim}
\bibfield{author}{\bibinfo{person}{Jiaming Song}, \bibinfo{person}{Chenlin
  Meng}, {and} \bibinfo{person}{Stefano Ermon}.}
  \bibinfo{year}{2021}\natexlab{}.
\newblock \showarticletitle{Denoising Diffusion Implicit Models}. In
  \bibinfo{booktitle}{\emph{International Conference on Learning
  Representations}}.
\newblock


\bibitem[Song et~al\mbox{.}(2025)]%
        {song2025idprotector}
\bibfield{author}{\bibinfo{person}{Yiren Song}, \bibinfo{person}{Pei Yang},
  \bibinfo{person}{Hai Ci}, {and} \bibinfo{person}{Mike~Zheng Shou}.}
  \bibinfo{year}{2025}\natexlab{}.
\newblock \showarticletitle{{IDProtector}: An Adversarial Noise Encoder to
  Protect Against ID-Preserving Image Generation}. In
  \bibinfo{booktitle}{\emph{Proceedings of the IEEE/CVF Conference on Computer
  Vision and Pattern Recognition}}. \bibinfo{pages}{3019--3028}.
\newblock


\bibitem[Van~Le et~al\mbox{.}(2023)]%
        {vanle2023antidreambooth}
\bibfield{author}{\bibinfo{person}{Thanh Van~Le}, \bibinfo{person}{Hao Phung},
  \bibinfo{person}{Thuan~Hoang Nguyen}, \bibinfo{person}{Quan Dao},
  \bibinfo{person}{Ngoc~N. Tran}, {and} \bibinfo{person}{Anh Tran}.}
  \bibinfo{year}{2023}\natexlab{}.
\newblock \showarticletitle{{Anti-DreamBooth}: Protecting Users from
  Personalized Text-to-Image Synthesis}. In
  \bibinfo{booktitle}{\emph{Proceedings of the IEEE/CVF International
  Conference on Computer Vision}}. \bibinfo{pages}{2116--2127}.
\newblock


\bibitem[Vaswani et~al\mbox{.}(2017)]%
        {vaswani2017attention}
\bibfield{author}{\bibinfo{person}{Ashish Vaswani}, \bibinfo{person}{Noam
  Shazeer}, \bibinfo{person}{Niki Parmar}, \bibinfo{person}{Jakob Uszkoreit},
  \bibinfo{person}{Llion Jones}, \bibinfo{person}{Aidan~N. Gomez},
  \bibinfo{person}{Lukasz Kaiser}, {and} \bibinfo{person}{Illia Polosukhin}.}
  \bibinfo{year}{2017}\natexlab{}.
\newblock \showarticletitle{Attention Is All You Need}. In
  \bibinfo{booktitle}{\emph{Advances in Neural Information Processing
  Systems}}, Vol.~\bibinfo{volume}{30}. \bibinfo{pages}{5998--6008}.
\newblock


\bibitem[Vu et~al\mbox{.}(2026)]%
        {vu2026antii2v}
\bibfield{author}{\bibinfo{person}{Duc Vu}, \bibinfo{person}{Anh Nguyen},
  \bibinfo{person}{Chi Tran}, {and} \bibinfo{person}{Anh Tran}.}
  \bibinfo{year}{2026}\natexlab{}.
\newblock \showarticletitle{{Anti-I2V}: Safeguarding Your Photos from Malicious
  Image-to-Video Generation}. In \bibinfo{booktitle}{\emph{Proceedings of the
  IEEE/CVF Conference on Computer Vision and Pattern Recognition}}.
  \bibinfo{pages}{37621--37631}.
\newblock


\bibitem[Wang et~al\mbox{.}(2023a)]%
        {wang2023semanticadversarial}
\bibfield{author}{\bibinfo{person}{Chenan Wang}, \bibinfo{person}{Jinhao Duan},
  \bibinfo{person}{Chaowei Xiao}, \bibinfo{person}{Edward Kim},
  \bibinfo{person}{Matthew Stamm}, {and} \bibinfo{person}{Kaidi Xu}.}
  \bibinfo{year}{2023}\natexlab{a}.
\newblock \showarticletitle{Semantic Adversarial Attacks via Diffusion Models}.
  In \bibinfo{booktitle}{\emph{British Machine Vision Conference}}.
\newblock


\bibitem[Wang et~al\mbox{.}(2024)]%
        {wang2024simac}
\bibfield{author}{\bibinfo{person}{Feifei Wang}, \bibinfo{person}{Zhentao Tan},
  \bibinfo{person}{Tianyi Wei}, \bibinfo{person}{Yue Wu}, {and}
  \bibinfo{person}{Qidong Huang}.} \bibinfo{year}{2024}\natexlab{}.
\newblock \showarticletitle{{SimAC}: A Simple Anti-Customization Method for
  Protecting Face Privacy against Text-to-Image Synthesis of Diffusion Models}.
  In \bibinfo{booktitle}{\emph{Proceedings of the IEEE/CVF Conference on
  Computer Vision and Pattern Recognition}}. \bibinfo{pages}{12047--12056}.
\newblock


\bibitem[Wang et~al\mbox{.}(2023b)]%
        {wang2023modelscope}
\bibfield{author}{\bibinfo{person}{Jiuniu Wang}, \bibinfo{person}{Hangjie
  Yuan}, \bibinfo{person}{Dayou Chen}, \bibinfo{person}{Yingya Zhang},
  \bibinfo{person}{Xiang Wang}, {and} \bibinfo{person}{Shiwei Zhang}.}
  \bibinfo{year}{2023}\natexlab{b}.
\newblock \showarticletitle{{ModelScope} Text-to-Video Technical Report}.
\newblock \bibinfo{journal}{\emph{arXiv preprint arXiv:2308.06571}}
  (\bibinfo{year}{2023}).
\newblock


\bibitem[Xing et~al\mbox{.}(2025)]%
        {xing2023dynamicrafter}
\bibfield{author}{\bibinfo{person}{Jinbo Xing}, \bibinfo{person}{Menghan Xia},
  \bibinfo{person}{Yong Zhang}, \bibinfo{person}{Haoxin Chen},
  \bibinfo{person}{Wangbo Yu}, \bibinfo{person}{Hanyuan Liu},
  \bibinfo{person}{Gongye Liu}, \bibinfo{person}{Xintao Wang},
  \bibinfo{person}{Ying Shan}, {and} \bibinfo{person}{Tien-Tsin Wong}.}
  \bibinfo{year}{2025}\natexlab{}.
\newblock \showarticletitle{{DynamiCrafter}: Animating Open-domain Images with
  Video Diffusion Priors}. In \bibinfo{booktitle}{\emph{Computer Vision--ECCV
  2024}} \emph{(\bibinfo{series}{Lecture Notes in Computer Science},
  Vol.~\bibinfo{volume}{15104})}. \bibinfo{publisher}{Springer},
  \bibinfo{pages}{399--417}.
\newblock


\bibitem[Xing et~al\mbox{.}(2023)]%
        {xing2023videodiffusionsurvey}
\bibfield{author}{\bibinfo{person}{Zhen Xing}, \bibinfo{person}{Qijun Feng},
  \bibinfo{person}{Haoran Chen}, \bibinfo{person}{Qi Dai}, \bibinfo{person}{Han
  Hu}, \bibinfo{person}{Hang Xu}, \bibinfo{person}{Zuxuan Wu}, {and}
  \bibinfo{person}{Yu-Gang Jiang}.} \bibinfo{year}{2023}\natexlab{}.
\newblock \showarticletitle{A Survey on Video Diffusion Models}.
\newblock \bibinfo{journal}{\emph{arXiv preprint arXiv:2310.10647}}
  (\bibinfo{year}{2023}).
\newblock


\bibitem[Yu et~al\mbox{.}(2024)]%
        {yu2024mfa}
\bibfield{author}{\bibinfo{person}{Hongwei Yu}, \bibinfo{person}{Jiansheng
  Chen}, \bibinfo{person}{Xinlong Ding}, \bibinfo{person}{Yudong Zhang},
  \bibinfo{person}{Ting Tang}, {and} \bibinfo{person}{Huimin Ma}.}
  \bibinfo{year}{2024}\natexlab{}.
\newblock \showarticletitle{Step Vulnerability Guided Mean Fluctuation
  Adversarial Attack against Conditional Diffusion Models}. In
  \bibinfo{booktitle}{\emph{Proceedings of the AAAI Conference on Artificial
  Intelligence}}, Vol.~\bibinfo{volume}{38}. \bibinfo{pages}{6791--6799}.
\newblock


\bibitem[Yu et~al\mbox{.}(2026)]%
        {yu2026medusa}
\bibfield{author}{\bibinfo{person}{Hongwei Yu}, \bibinfo{person}{Daoqing Zha},
  \bibinfo{person}{Xinlong Ding}, \bibinfo{person}{Jiawei Li},
  \bibinfo{person}{Junbao Zhuo}, \bibinfo{person}{Qiankun Liu},
  \bibinfo{person}{Huimin Ma}, {and} \bibinfo{person}{Jiansheng Chen}.}
  \bibinfo{year}{2026}\natexlab{}.
\newblock \showarticletitle{{MEDUSA}: Motion Elimination in Diffusion Using
  Spectral Attack}. In \bibinfo{booktitle}{\emph{Proceedings of the
  International Conference on Machine Learning}}.
\newblock


\bibitem[Zeng et~al\mbox{.}(2025)]%
        {zeng2024advi2i}
\bibfield{author}{\bibinfo{person}{Yaopei Zeng}, \bibinfo{person}{Yuanpu Cao},
  \bibinfo{person}{Bochuan Cao}, \bibinfo{person}{Yurui Chang},
  \bibinfo{person}{Jinghui Chen}, {and} \bibinfo{person}{Lu Lin}.}
  \bibinfo{year}{2025}\natexlab{}.
\newblock \showarticletitle{{AdvI2I}: Adversarial Image Attack on
  Image-to-Image Diffusion Models}. In \bibinfo{booktitle}{\emph{Proceedings of
  the 42nd International Conference on Machine Learning}}
  \emph{(\bibinfo{series}{Proceedings of Machine Learning Research},
  Vol.~\bibinfo{volume}{267})}. \bibinfo{pages}{74037--74050}.
\newblock


\bibitem[Zhang et~al\mbox{.}(2023)]%
        {zhang2023i2vgenxl}
\bibfield{author}{\bibinfo{person}{Shiwei Zhang}, \bibinfo{person}{Jiayu Wang},
  \bibinfo{person}{Yingya Zhang}, \bibinfo{person}{Kang Zhao},
  \bibinfo{person}{Hangjie Yuan}, \bibinfo{person}{Zhiwu Qin},
  \bibinfo{person}{Xiang Wang}, \bibinfo{person}{Deli Zhao}, {and}
  \bibinfo{person}{Jingren Zhou}.} \bibinfo{year}{2023}\natexlab{}.
\newblock \showarticletitle{{I2VGen-XL}: High-Quality Image-to-Video Synthesis
  via Cascaded Diffusion Models}.
\newblock \bibinfo{journal}{\emph{arXiv preprint arXiv:2311.04145}}
  (\bibinfo{year}{2023}).
\newblock


\bibitem[Zhao et~al\mbox{.}(2024)]%
        {zhao2024protective}
\bibfield{author}{\bibinfo{person}{Zhengyue Zhao}, \bibinfo{person}{Jinhao
  Duan}, \bibinfo{person}{Kaidi Xu}, \bibinfo{person}{Chenan Wang},
  \bibinfo{person}{Rui Zhang}, \bibinfo{person}{Zidong Du}, \bibinfo{person}{Qi
  Guo}, {and} \bibinfo{person}{Xing Hu}.} \bibinfo{year}{2024}\natexlab{}.
\newblock \showarticletitle{Can Protective Perturbation Safeguard Personal Data
  from Being Exploited by {Stable Diffusion}?}. In
  \bibinfo{booktitle}{\emph{Proceedings of the IEEE/CVF Conference on Computer
  Vision and Pattern Recognition}}. \bibinfo{pages}{24398--24407}.
\newblock


\end{thebibliography}

\clearpage
\appendix
\onecolumn
\section{Loss-Weight Sensitivity}
\label{app:loss-weight-sensitivity}
\begin{figure}[htbp]
  \centering
  \includegraphics[width=0.96\textwidth]{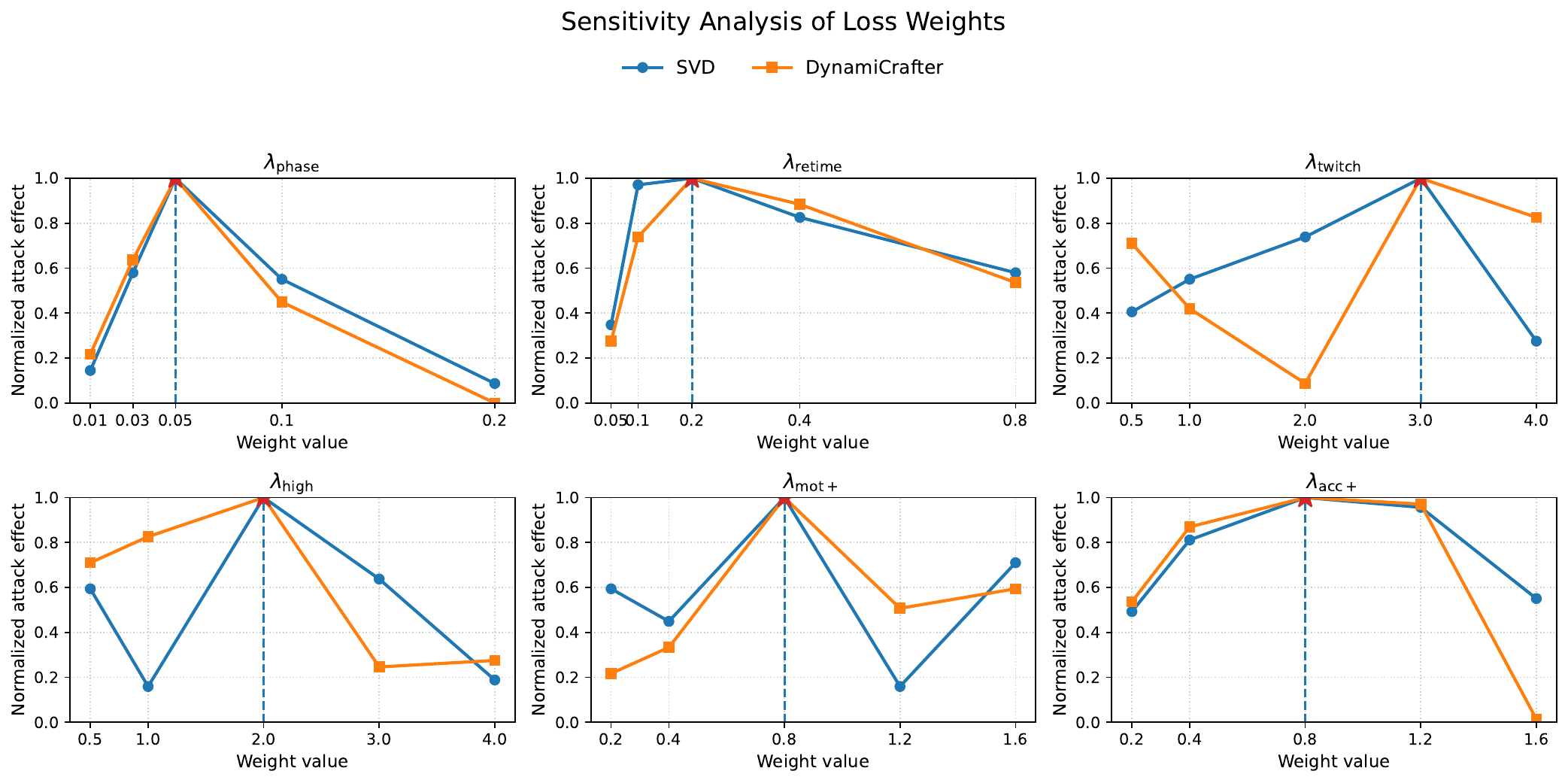}
  \caption{Sensitivity analysis of the six loss weights on SVD and
  DynamiCrafter. Each subplot varies one weight while holding the others at
  their default values. The dashed vertical line denotes the selected value,
  and stars mark the corresponding attack effects.}
  \Description{Six line plots show loss-weight sensitivity on Stable Video
  Diffusion and DynamiCrafter. Each plot varies one loss weight and reports
  normalized attack effect. Dashed vertical lines and star markers identify
  the selected weight and its corresponding effects.}
  \label{fig:weight-sensitivity}
\end{figure}
We evaluate the sensitivity of PhaseBreak to each loss weight while keeping
the remaining weights fixed at their default values. Figure~\ref{fig:weight-sensitivity}
reports the normalized attack effect on Stable Video Diffusion (SVD) and
DynamiCrafter across the candidate values. The dashed vertical lines and star
markers indicate the values used in the main experiments.

\section{Protected Input Image Quality}
\label{app:input-image-quality}
We evaluate the visual fidelity of the protected input images using PSNR,
SSIM, and LPIPS. As shown in Table~\ref{tab:input-image-quality}, PhaseBreak
maintains high pixel- and structure-level similarity to the corresponding
clean images under both target models, while keeping perceptual distances
low.

\begin{table}[ht]
  \centering
  \small
  \setlength{\tabcolsep}{8pt}
  \renewcommand{\arraystretch}{1.05}
  \caption{Visual quality of protected input images. Higher PSNR and SSIM and
  lower LPIPS indicate better preservation of the clean input.}
  \label{tab:input-image-quality}
  \begin{tabular}{@{}lccc@{}}
    \toprule
    \textbf{Target Model} & \textbf{PSNR}$\uparrow$ &
    \textbf{SSIM}$\uparrow$ & \textbf{LPIPS}$\downarrow$ \\
    \midrule
    Stable Video Diffusion & 37.5089 & 0.9290 & 0.0678 \\
    DynamiCrafter          & 38.0340 & 0.9385 & 0.0507 \\
    \bottomrule
  \end{tabular}
\end{table}

\end{document}